\pdfoutput=1

\documentclass[11pt]{article}

\usepackage[hang, flushmargin]{footmisc}
\usepackage{EMNLP2022}
\usepackage{footnotebackref}
\usepackage{times}
\usepackage{latexsym}
\usepackage{amsmath,amsfonts,amssymb}
\usepackage{amsthm}
\usepackage{mathtools}
\usepackage{graphicx}
\usepackage{color}
\usepackage[normalem]{ulem}
\usepackage{diagbox}
\usepackage{makecell}
\usepackage{multicol}
\usepackage{alltt}
\usepackage{multirow}
\usepackage{tabularx}
\usepackage{paralist}
\usepackage{array}
\usepackage{setspace}
\usepackage{multirow}
\usepackage{tabularx}
\usepackage{booktabs}
\usepackage{makecell}
\usepackage{outlines}
\usepackage[misc]{ifsym} 
\usepackage{lipsum}

\usepackage[T1]{fontenc}

\usepackage[utf8]{inputenc}

\usepackage{microtype}

\usepackage[normalem]{ulem}
\usepackage{inconsolata}
\usepackage{xspace}
\def\formername{\textsc{TransNormer}\xspace}
\def\localattentionname{\textsc{DiagAttention}\xspace}
\def\attentionname{\textsc{NormAttention}\xspace}

\usepackage{listings}
\usepackage{color}

\definecolor{dkgreen}{rgb}{0,0.6,0}
\definecolor{gray}{rgb}{0.5,0.5,0.5}
\definecolor{mauve}{rgb}{0.58,0,0.82}

\def\etc{\emph{etc.}}
\def\ie{\emph{i.e., }}

\newcommand*{\tran}{^{\mkern-1.5mu\mathsf{T}}}

\newcommand\blfootnote[1]{%
  \begingroup
  \renewcommand\thefootnote{}\footnote{#1}%
  \addtocounter{footnote}{-1}%
  \endgroup
}

\title{The Devil in Linear Transformer}

\author{
{\normalsize
$^{\star}$Zhen Qin$^{1}$,
$^{\star}$Xiaodong Han$^{1}$,
Weixuan Sun$^{2,3}$,
Dongxu Li$^{2}$,
}\\
{\normalsize
\textbf{Lingpeng Kong$^{4,5}$,
Nick Barnes$^{2}$,
$^\textrm{\Letter}$Yiran Zhong$^{4}$}
}\\
$^{1}$SenseTime Research, $^{2}$Australian National University,\\ $^{3}$OPPO Research Institute, $^{4}$Shanghai AI Laboratory, $^{5}$The University of Hong Kong\\
\ \ \texttt{https://github.com/OpenNLPLab/Transnormer} 
}

\begin{document}
\maketitle

\begin{abstract}
Linear transformers aim to reduce the quadratic space-time complexity of vanilla transformers. However, they usually suffer from degraded performances on various tasks and corpora.
In this paper, we examine existing kernel-based linear transformers and identify two key issues that lead to such performance gaps: 
1) \emph{unbounded gradients} in the attention computation adversely impact the convergence of linear transformer models;
2) \emph{attention dilution} which trivially distributes attention scores over long sequences while neglecting neighbouring structures.
To address these issues, we first identify that the \emph{scaling} of attention matrices is the devil in unbounded gradients, which turns out unnecessary in linear attention as we show theoretically and empirically.
To this end, we propose a new linear attention that replaces the scaling operation with a normalization to stabilize gradients.
For the issue of attention dilution, we leverage a \emph{diagonal attention} to confine attention to only neighbouring tokens in early layers.
%
%
Benefiting from the stable gradients and improved attention, our new linear transformer model, \formername, demonstrates superior performance on text classification and language modeling tasks, as well as on the challenging Long-Range Arena benchmark, surpassing vanilla transformer and existing linear variants by a clear margin while being significantly more space-time efficient. The code is available at \href{https://github.com/OpenNLPLab/Transnormer}{\formername}.
\blfootnote{\noindent $^{\star}$Equal contribution. $^\textrm{\Letter}$ The corresponding author (Email: \textit{zhongyiran@gmail.com}). This work was done when Weixuan Sun and Yiran Zhong were in the SenseTime Research.}
\end{abstract}

\section{Introduction}
Transformer models show great performance on a wide range of natural language processing and computer vision tasks~\cite{zhen2022cosformer,sun2022vicinity,cheng2022implicit,cheng2022deep,zhou2022avs}.
One issue of the vanilla transformer model lies in its quadratic space-time complexity with respect to the input length.
Various prior works attempt to alleviate this inefficiency~\cite{zaheer2020big,beltagy2020longformer,tay2020sparse,kitaev2020reformer,child2019generating,liu2022neural,sun2022vicinity}.
In this work, we focus on a particular subset of these methods, known as \emph{kernel-based linear transformers}~\cite{choromanski2020rethinking,wang2020linformer,katharopoulos2020transformers,peng2021random,zhen2022cosformer} considering their desirable linear space-time complexity.
%


Despite their space-time efficiency, linear transformers are not always in favor for practical adoption, largely due to the degraded performance than the vanilla model.
To address this issue, we take a close look at existing kernel-based linear transformers and identify \textbf{\emph{two}} deficiencies that lead to such a performance gap.
\begin{figure}[t]    
    \begin{center}    
    {\includegraphics[width=1.0\linewidth]{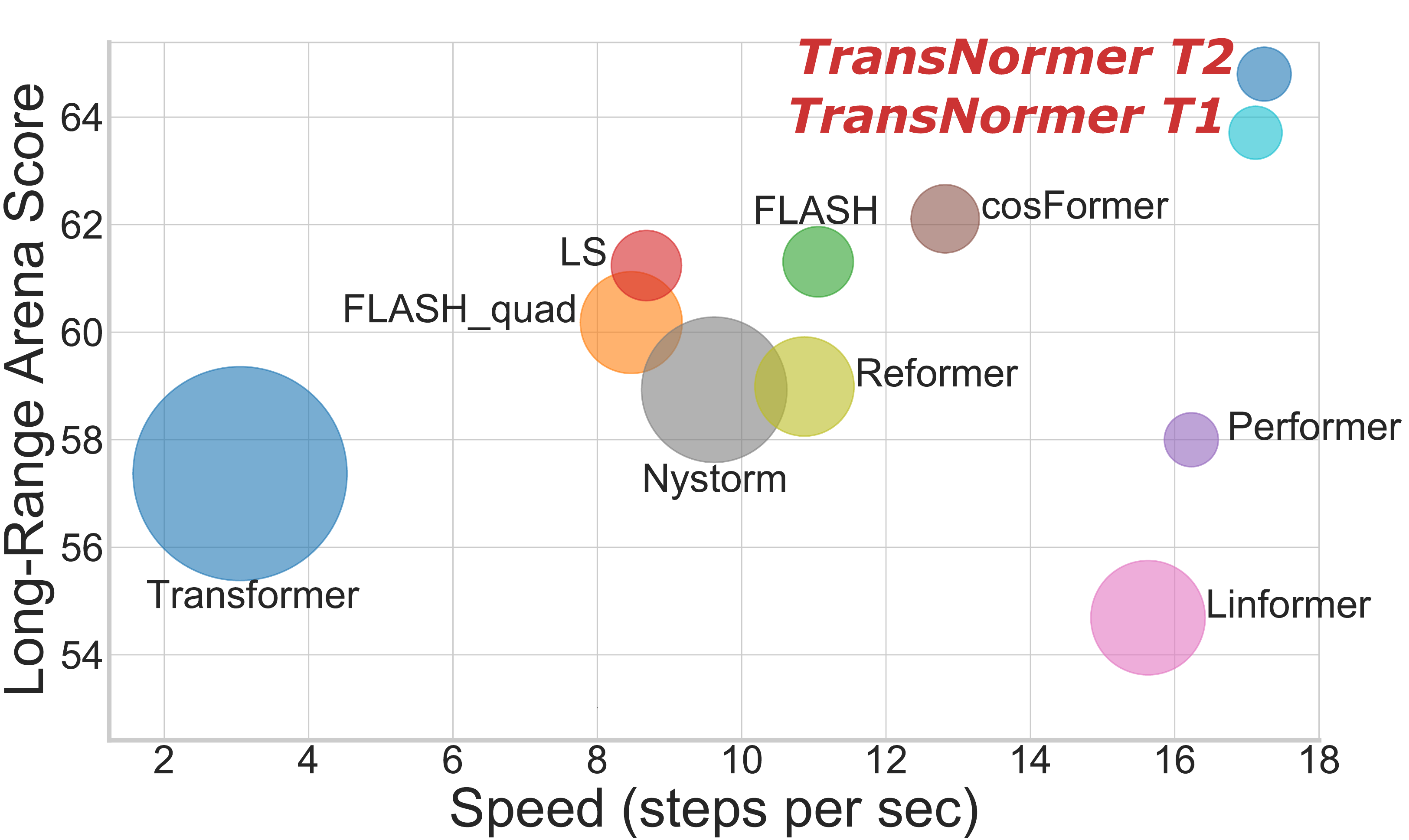}}     
    \end{center} 
    \caption{\formername has smaller memory footprints (circle sizes) and produces clearly favorable speed (x-axis) and overall scores (y-axis), when evaluated on the challenging Long-Range Arena benchmark than the vanilla transformer and other competing methods.
    }   
    \label{fig:score speed mem} 
\end{figure}

\noindent
\textbf{{Unbounded gradients.}} Most existing linear transformers inherit attention formulation from the vanilla transformer, which scales attention scores to ensure they are bounded within $[0,1]$.
However, we theoretically show that such a scaling strategy renders unbounded gradients for linear transformer models.
As a result, the unbounded gradients empirically lead to unstable convergence as our preliminary experiments suggest.
%
\noindent \\
\textbf{{Attention dilution.}}
%
Previous works~\cite{titsias2016one,jang2016categorical,gao2017properties,zhen2022cosformer,sun2022vicinity,sun2022locality} suggest that in vanilla transformer, softmax attention maps tend to be local.
In contrast, as shown in Fig~\ref{fig:attention matrix}, we observe that linear transformers often trivially distribute attention scores over the entire sequence even in early layers.
Due to this issue, which we refer as \emph{attention dilution}, important local information is less well preserved in linear models, resulting in inferior performance. 
%
%
This negative impact of attention dilution is also evidenced by the performance drop in our controlled experiments if partly replacing vanilla attention in transformer layers with linear attention ones.

To mitigate these issues, we propose a linear transformer model, called \formername, which shows better performance than vanilla transformer on a wide range of task while being significantly faster during runtime, as shown in Fig.~\ref{fig:score speed mem}.

%
%

%
%
%

To avoid the unbounded gradients, we introduce \attentionname, which gets rid of scaling over attention matrices while appending an additional normalization only \emph{after} the attention layer.
The choice of the normalization operator is unrestricted, for example, LayerNorm~\cite{1607.06450} or RMSNorm~\cite{zhang-sennrich-neurips19} both serve the purpose.
We show empirical results demonstrating that with \attentionname, the gradients are more stable during training, which in turn leads to more consistent convergence.



%
%

To alleviate the attention dilution issue, we modify the vanilla attention and allow each token to only attend to its neighbouring tokens, resulting in a \emph{diagonal} attention.
To mimic the behaviors on local semantics of the vanilla transformer, we employ the diagonal attention on early layers while using \attentionname for later ones.
In this way, we encourage the model to capture both local and global language context.
Note that our diagonal attention can be efficiently computed such that the overall linear space-time complexity of \formername is preserved.


We perform extensive experiments on standard tasks, where \formername demonstrates \emph{lower language modeling perplexities} on WikiText-103 and overall \emph{higher text classification accuracy} on GLUE than vanilla model and other competing methods.
In addition, on the challenging Long-Range Arena benchmark, \formername also shows favorable results while being \emph{faster and more scalable} with longer inputs during both training and inference time.


%




\section{Background and related work}
We first briefly review vanilla transformer \cite{vaswani2017attention} and its efficient variants.
The key component of transformers is the self-attention, which operates on query $\mathbf{Q}$, key $\mathbf{K}$ and value $\mathbf{V}$ matrices; each of them is the image of a linear projection taking $\mathbf X\in \mathbb R^{n\times d}$ as input:
\begin{equation}
\textstyle
\small
\begin{gathered}
    \mathbf{Q} = \mathbf{X}\mathbf{W}_Q  ,
 \mathbf{K} =\mathbf{X}\mathbf{W}_K,
\mathbf{V} =\mathbf{X}\mathbf{W}_V \in \mathbb{R}^{n\times d},
\end{gathered}
\end{equation}
with $n$ the input length, $d$ the hidden dimension. 
The output $\mathbf{O}\in \mathbb R^{n\times d}$ is formulated as:
\begin{equation}
\textstyle
\small
\mathbf{O} = \mathrm{Softmax}(\mathbf{Q} \mathbf{K}\tran / \sqrt{d}) \mathbf{V},
\end{equation}
where the $\mathrm{Softmax}(\cdot)$ step renders quadratic space-time complexity with respect to the input length, making it prohibitive for vanilla transformer to scale to long input sequences.
%
To address this issue, numerous efficient transformers have been explored in the literature.
These methods can be generally categorized into two families, \ie \emph{pattern based methods} and \emph{kernel based methods}.

Pattern based methods~\cite{zaheer2020big,beltagy2020longformer,tay2020sparse,kitaev2020reformer,child2019generating} sparsify the attention calculation with handcrafted or learnable masking patterns.
Kernel-based methods adopt kernel functions to decompose softmax attention, which reduces the theoretical space-time complexity to linear. In this paper, we refer the kernel-based variants as linear transformers for simplicity.

In the kernel-based methods~\cite{choromanski2020rethinking,katharopoulos2020transformers,peng2021random,zhen2022cosformer,zheng2022linear,wang2020linformer}, a kernel function $\phi(\cdot)$ maps queries and keys to their hidden representations. Then the output of the linear attention can be rewritten as:
\begin{equation}
\small
\begin{aligned}
\mathbf{O}&=\mathbf \Delta^{-1} \phi(\mathbf{Q}) [\phi(\mathbf{K})\tran\mathbf V ],\\
\mathbf{\Delta} &= \mathrm{diag}(\phi(\mathbf{Q}) [\phi(\mathbf{K})\tran {\mathbf 1}_n]). 
\end{aligned}
\end{equation}
where the product of keys and values are computed to avoid the quadratic $n \times n$ matrix.
Existing methods mainly differ in the design of kernel functions.
For example, \citet{choromanski2020rethinking} and \citet{katharopoulos2020transformers} adopt activation function $1+elu$ to process query and key.
\citet{wang2020linformer} 
assumes attention matrices are low-rank. \citet{peng2021random} and \citet{zheng2022linear} approximate softmax under constrained theoretical bounds. \citet{zhen2022cosformer} propose a linear alternative to the attention based on empirical properties of the softmax function. 

These methods focus on either approximating or altering the softmax operator while preserving its properties.
Compared with the vanilla transformer, these methods often trade performance for efficiency, usually resulting in worse task performance.  
In this paper, we argue that there are two essential reasons leading to such a performance gap, discussed in detail as follows.

\begin{figure*}[t!]
   \begin{center}
   {\includegraphics[width=1\linewidth]{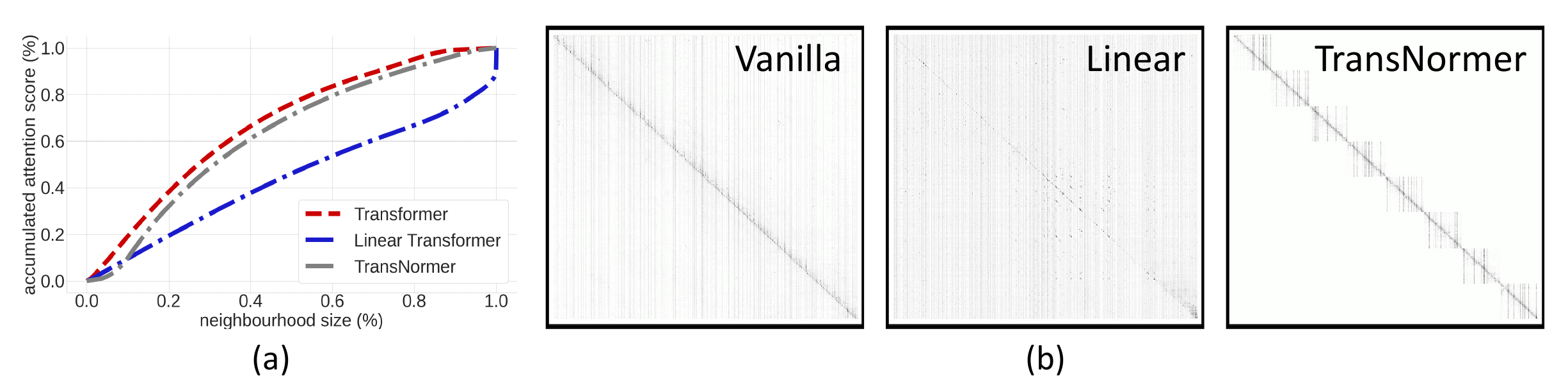}} 
   \end{center}
  \vspace{-3mm}
\caption{(a): Comparison of locally accumulated attention scores of different transformer variants. The x-axis denotes ratio of neighbourhood size relative to the input length; the y-axis denotes accumulated attention scores inside this neighbourhood for the centering token.
The curve for the vanilla transformer model increases more sharply, indicating that the attention scores are more concentrated. Our model greatly alleviates the attention dilution issue for linear models.
(b): Qualitative comparison of attention matrices in early model layers. The proposed \formername produces more similar patterns to the original vanilla transformer, benefiting to better capture local-global language context, while the linear model suffers clearly from the issue of attention dilution and gets distracted by distant tokens in early layers.
}
   \label{fig:attention matrix}
\end{figure*}

\section{The devil in linear attention}
\label{section:drawbacks}
In this section, we motivate the design principles of 
\formername by providing theoretical evidence for the unbounded gradients, and empirical results showing the adverse influence of attention dilution.
\subsection{Unbounded gradients}
Few work on linear transformers analyzes their gradients during training.
Our first key observation is that \emph{kernel-based linear attention suffer from unbounded gradients, causing unstable convergence during training}.
In the following, we highlight the main theoretical results while referring readers to Appendix~\ref{sec:proof} for the full derivation.


Consider a self-attention module, either vanilla or linear attention. Its attention matrix $\mathbf P \in \mathbb R^{n\times n}$ can be represented in the following unified form
\footnote{Here we assume that $f(s_{ij})\ge 0$, the conclusion is satisfied in most cases.}:
\begin{equation}
\small
\label{equation:attention matrix}
 p_{ij}=\frac{f(s_{ij})}{\sum_{k=1}^n f(s_{ik})},
f:\mathbb R \to \mathbb R.
\end{equation}

Vanilla and linear attention differ mainly in their computation of token-wise similarities $s_{ij}$\footnote{Note that $s_{ij}$ is not directly computed in linear attention, but can still be represented in this unified form, see Appendix~\ref{sec:proof} for more detailed derivation}. 
In vanilla attention, $s_{ij}$ is computed as:
\begin{equation}
\small
s_{ij}= \mathbf q_i \tran \mathbf k_j  /\sqrt d,f(x)=\exp(x),
\end{equation}
while for linear attentions, $s_{ij}$ can be decomposed using a kernel function $\phi$, such that:
\begin{equation}
\small
s_{ij}=  \phi(\mathbf q_i) \tran  \phi(\mathbf k_j) ,f(x)=x.
\end{equation}

Given the above definitions, the gradients of the attention matrix $\mathbf{P}$ is derived as:
\begin{equation}
\small
\frac{\partial p_{ij}}{\partial s_{ik}}
= \frac{f'(s_{ik})}{f(s_{ik})}\left(1_{j=k} p_{ij} - p_{ij}p_{ik} \right)
\end{equation}
Therefore, for the vanilla attention,  the partial derivative $\frac{\partial  p_{ij}}{\partial s_{ik}}$ is:
\begin{equation}
\label{eq:gradient1}
\small
\begin{aligned}
f'(x)&= \exp(x)=f(x) \\
\frac{\partial p_{ij}}{\partial s_{ik}}
&=1_{j=k} p_{ij} - p_{ij}p_{ik}\\
&=\begin{cases}
p_{ik} - p_{ik}p_{ik}\in[0, 1/4] & j=k\\
-p_{ij}p_{ik} \in [-1/4, 0] & j\neq k
\end{cases}
\end{aligned}
\end{equation}
and it is bounded as:
\begin{equation}
\small
\label{vanilla}
\left| \frac{\partial p_{ij}}{\partial s_{ik}}\right|  \le \frac 1 4.
\end{equation}
However, for linear attentions, we have:
\begin{equation}
\label{eq:gradient2}
\small
\begin{aligned}
f'(x)&=1\\
\frac{\partial p_{ij}}{\partial s_{ik}}
&=\frac{1}{s_{ik}}\left(1_{j=k} p_{ij} - p_{ij}p_{ik} \right)\\
&=\begin{cases}
\frac{1}{s_{ik}}\left(p_{ik} - p_{ik}p_{ik}\right) & j=k\\
\frac{1}{s_{ik}}(-p_{ij}p_{ik})  & j\neq k
\end{cases}
\end{aligned}
\end{equation}
and\footnote{A detailed proof of the upper bound can be found at Appendix \ref{sec:bound}.}
\begin{equation}
\small
\label{linear}
\left| \frac{\partial p_{ij}}{\partial s_{ik}}\right |  \le \frac 1 {4 |s_{ik}|}.
\end{equation}
Since
$|s_{ik}|^{-1}=  |\phi(\mathbf q_i) \phi(\mathbf k_j) \tran|^{-1}$ can be arbitrarily large, the gradient of linear attention has no upper bound. On the other hand, we can also show that the gradient of linear attention has no lower bound\footnote{The proof can be found in Appendix \ref{sec:linear}.}:
\newtheorem{linear}{Proposition}[section]
\begin{linear}
\label{gradient:linear}
$\forall M>0$, there exists $\mathbf q_i, \mathbf k_j\in \mathbb R^{d}, j=1,\ldots, n $, such that:
\begin{equation}
\left| \frac{\partial p_{ij}}{\partial s_{ik}}\right |> M.
\end{equation}
\end{linear} 
The unbounded gradients lead to less stable optimization and worse convergence results in our preliminary studies.

\subsection{Attention dilution}
It is a known property of vanilla attention to emphasize on neighbouring tokens ~\cite{titsias2016one,zhen2022cosformer}.
However, this property does not directly inherit to the linear transformer variants.
%

To quantify the attention dilution issue, we introduce a metric called \emph{locally accumulated attention} score, which measures how much attention scores are distributed within the local neighbourhood of a particular token.

For an input sequence of length $N$, consider a local neighbourhood $\{x_{start},...,x_i...,x_{end}\}$ \emph{centering} around token $x_i$ of total length $r\cdot N$, with $r$ the ratio relative to the total input, the locally accumulated attention score for token $x_i$ is defined as $l(i,r,N)=p_{i,start}+...+p_{i,end}$.
A higher score indicates the particular attention layer concentrates on the local neighbourhood, while a lower score tends to indicate the issue of attention dilution, where scores are distributed more evenly to local and distant tokens. For example, $l(i, 0.4, N)=0.6$ means that that 40\% of the neighbors around $i$'th token contribute 60\% of the attention score. 

In Fig.~\ref{fig:attention matrix} (a), we compare locally accumulated attention scores (y-axis) for vanilla transformer and linear transformer, with varying sizes of neighbourhood by ratio (x-axis).
We show the average score over each position across the entire sequence.
It can be seen that the area under the vanilla model curve is significantly larger than that of the linear model.
This provides evidence that the vanilla attention is more concentrated locally, while the linear transformer suffers from the issue of attention dilution.
This is further qualitatively supported by Fig.~\ref{fig:attention matrix}~(b), where the attention maps for vanilla model are more concentrated than the linear model.

\section{Method}
Based on the aforementioned observations, we propose a new linear transformer network called \formername that addresses the above two limitations of current linear transformers. The overall architecture is shown in Fig.~\ref{fig:network}. 

\subsection{The overall architecture}
%
Vanilla attention suffers less in attention dilution while linear attention is more efficient and scalable on longer sequences. This motivate us to design a method that exploits the best of the both worlds by using these mechanisms in combined.

Specifically, our network consists of two types of attention: \localattentionname for the early stage of the model and \attentionname for the later stage.
The former addresses the attention dilution issue and the later aims to stabilize training gradients.
Note that by properly reshaping the inputs, the diagonal attention can be efficiently computed in linear space-time, thus preserving the overall linear complexity.
%


\subsection{\attentionname}
\begin{table}[!ht]
    \small
    \setlength{\tabcolsep}{0.9cm}
\caption{\textbf{Ablation of linear attention with scaling operation.} Directly removing scaling operation \ie the denominator in Eq.~\ref{equation:attention matrix}, leads to significant performance drop. Our normalization strategy achieves better result.}
    \label{denorm}
    \centering
    \begin{tabular}{lc}
    \toprule
        method & ppl(val) \\ \midrule
        $1+elu$ & 4.98 \\ 
        $1+elu$ w/o scaling & 797.08 \\
        \attentionname  & 4.94 \\
        \bottomrule
    \end{tabular}
\end{table}

As proved in Sec.~\ref{section:drawbacks}, the scaling operation, \ie the denominator in Eq.~\ref{equation:attention matrix},
in the linear transformers hinder the optimization due to the unbounded gradients. To solve this issue, we propose to remove the scaling operation in the linear transformers. However, as shown in Table.~\ref{denorm}, directly removing the scaling operation leads to critical performance drop since the attention map becomes unbounded in the forward pass.
Therefore, an alternative is required to bound both attention maps during forward and their gradients during backward passes in linear attentions.

Our proposed solution is simple yet effective.
Given a linear attention, the attention without scaling can be formulated as:
\begin{equation}
\small
\mathbf O= \mathbf Q (\mathbf K\tran \mathbf V).
\end{equation}
We empirically find that we can apply an arbitrary normalization on this attention to bound it, which leads to our \attentionname as:
\begin{equation}
\small
\mathbf O_{\text{norm}}= \mathrm{XNorm}(\mathbf Q (\mathbf K\tran \mathbf V)),
\end{equation}
where the $\mathrm{XNorm}$ can be Layernorm\citep{1607.06450} or RMSNorm \citep{zhang-sennrich-neurips19} and \etc~We use the RMSNorm in our experiments as it is slightly faster than other options.

It can be proved that the gradients of \attentionname is bounded by\footnote{The full derivation can be found in Appendix~\ref{sec:proof}.}:
\begin{equation}
\small
\left|\frac{\partial \mathcal L}{\partial s_{ij}} \right| \le \frac{3c_1 c_2d}{2\sqrt{\epsilon}} <\infty,
\end{equation}
where $\mathcal L$ is the loss function, $\epsilon$ is the small constant that used in RMSNorm, $d$ is the embedding dimension and
\begin{equation}
\small
\begin{aligned}
c_1&= \max_{i=1}^n \| \nabla_{\mathbf O_i}\mathcal L \|_2 < \infty \\
c_2&= \max_{i=1}^n \|\mathbf V_i \|_2 <\infty
\end{aligned}
\end{equation}

To demonstrate the gradients stability of the \attentionname, we compare the relative standard deviation of gradients during each training iterations to other linear transformers and vanilla transformer.
Specifically, we train our model for 50k iterations with RoBERTa architecture on the WikiText103~\cite{merity2017pointer} and obtain the relative standard deviation of all iterations' gradients. 
As shown in Table~\ref{table: gradients}, existing linear methods \cite{choromanski2020rethinking,katharopoulos2020transformers} have substantially higher deviations compared to vanilla attention, which leads to inferior results.
The \attentionname produces more stable gradients, which validates the effectiveness of our method.

\begin{table}[!ht]
    \caption{\textbf{Relative standard deviation of training gradients over 50k iterations.} Our proposed \attentionname provides more stable gradients which are closer to vanilla transformer.}
    \centering
    \small
    \label{table: gradients}
    \begin{tabular}{lc}
    \toprule
        method &  \makecell{Relative Standard \\ Deviation} \\
        \midrule
        $1+elu$ \cite{katharopoulos2020transformers} & 0.58 \\ 
        Performer\cite{choromanski2020rethinking} & 0.47 \\ 
        Vanilla\cite{vaswani2017attention} & 0.25 \\
        \attentionname & 0.20 \\ 
        \bottomrule
    \end{tabular}
\end{table}

\begin{figure}[!t]
   \begin{center}
   {\includegraphics[width=0.95\linewidth]{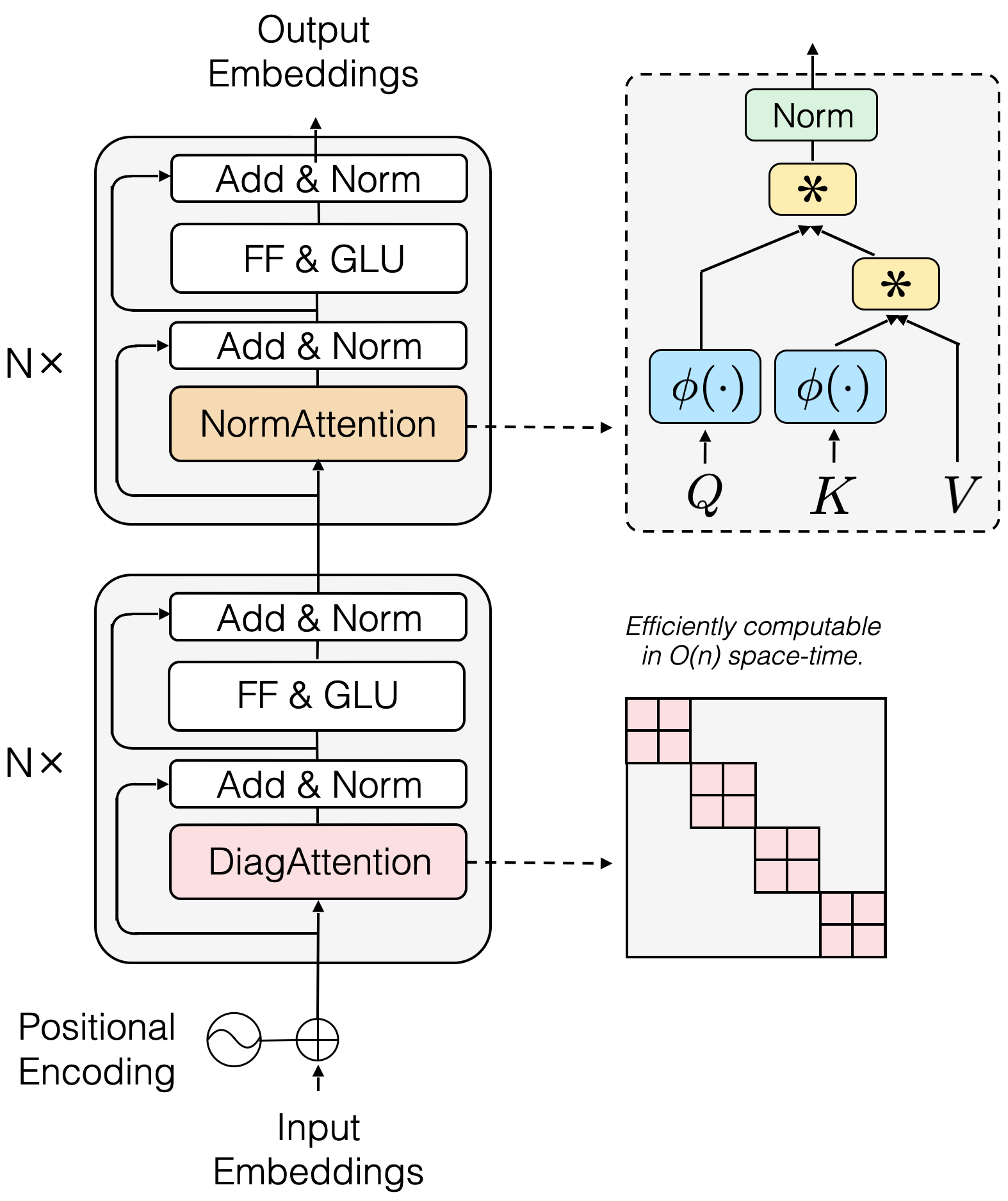}} 
   \end{center}
\caption{Architecture overview of the proposed \formername. In the early stages, we leverage \localattentionname, where attention is only calculated inside the blocks to enforce neighbouring focus. In late stages,  \attentionname assists to obtain a more stable gradients in linear complexity.}
   \label{fig:network}
\end{figure}



\subsection{\localattentionname}

To better understand the design principles, we show in Table~\ref{table pilot locality} that by replacing partial layers of linear transformers with vanilla attention, the performance on language modeling is evidently improved.
The results also suggest that capturing more local information in early layers are more helpful than otherwise.

\begin{table}
\centering 
\small
\setlength{\tabcolsep}{5mm}
\caption{\textbf{Ablation on attention dilution issue.} We implement all structures under the same setting: Vanilla \cite{vaswani2017attention}, $1+elu$ \cite{katharopoulos2020transformers}.
}
\begin{tabular}{ccc}
    \toprule
     \makecell{\textbf{Early layers}} & \makecell{\textbf{Later layers}} & \textbf{ppl} (val) \\ 
    \midrule
        $1+elu$ & $1+elu$ & 4.98 \\ 
        $1+elu$ & Vanilla & 3.90 \\
        Vanilla  & $1+elu$ & 3.76\\ 
  \bottomrule
\end{tabular}
  \label{table pilot locality}
\end{table}


To this end, we leverage \emph{none-overlapped} block-based strategy to reduce the space-time complexity of the vanilla attention. Based on the observation in Fig.~\ref{fig:attention matrix}, we utilize a strict diagonal blocked pattern to constraint the attention in a certain range. Since the attentions are calculated inside each block, the computation complexity of our diagonal attention is $O(nwd)$, where $n$ is sequence length , $w$ is the block size and $d$ is feature dimension. When $d\ll n$, the complexity scales linearly respect to the sequence length $n$. In subsequent sections, we use \localattentionname to refer to Diagonal attention.

We empirically find that applying \localattentionname to the later stages of a model hurts the performance as shown in Table.~\ref{order_v1}. 
It indicates that the model requires a global field of view in the later layers, which also justifies our choices of \attentionname in later layers of \formername.

\section{Experiments}
\label{sec:exp}
In this section, we compare our method to other linear transformers and the vanilla transformer on autoregressive language modeling, bidirectional language modeling as well as the Long Range Arena benchmark~\citep{tay2020long}. We also provide an extensive ablation study to vindicate our choice in designing the \formername.

We validate our method on two variants of the \formername. The \formername T1 uses the ReLA attention~\citep{zhang-etal-2021-sparse} in the \localattentionname and the \emph{elu} as the activation function in the \attentionname. The \formername T2 uses the $\mathtt{Softmax}$ attention~\citep{vaswani2017attention} in the \localattentionname and the \emph{1+elu} as the activation function in the \attentionname.


For experiments, we first study the autoregressive language modeling on WikiText-103 \citep{merity2017pointer} in section \ref{sec:alm}. 
Then in section \ref{sec:blm} we test our method on bidirectional language modeling, which is pre-trained on WikiText-103 \citep{merity2017pointer} and then fine-tuned on several downstream tasks from the GLUE benchmark \citep{wang2018glue}. 
Finally, we test \formername on the Long-Range Arena benchmark \citep{tay2020long} to evaluate its ability in modeling long-range dependencies and efficiencies in section \ref{sec:lra}.

\subsection{Settings}
We implement our models in the \textit{Fairseq} framework \citep{ott2019fairseq} and train them on 8 V100 GPUS. We use the same training configuration for all competitors and we list detailed hyper-parameters in Appendix \ref{lm:config}. We choose the FLASH-quad, FLASH \citep{hua2022transformer}, Transformer-LS \citep{zhu2021longshort}, Performer \citep{choromanski2020rethinking}, 1+elu \citep{katharopoulos2020transformers} as our main competing methods.

For the autoregressive language modeling, we use 6 decoder layers (10 layers for the FlASH/FLASH-quad) as our base model and all models are trained on the WikiText-103 dataset \citep{merity2017pointer} for 100K steps with a learning rate of $0.005$. We use the perplexity (PPL) as the evaluation metric. 

For the bidirectional language modeling, we choose the RoBERTa base \citep{liu2019RoBERTa} for all methods. It consists of 12 encoder layers (24 layers for the FLASH and FLASH-quad to match the number of parameters). All models are pre-trained on the WikiText-103~\citep{merity2017pointer} for 50K steps with lr=0.005 and fine-tuned on the GLUE dataset~\citep{wang2018glue}. We use different learning rates among {1e-5, 3e-5, 6e-5, 1e-4} and choosing the best result after fine-tuning for 3 epochs. 

For the Long-Range Arena benchmark, to make sure it reflect the practical speed in Pytorch platform, we re-implement the benchmark in Pytorch. We adopt the same configuration from the Skyformer~\cite{Skyformer} and make sure all models have a similar parameter size. We use the same training hyper parameters for all models as well. 


\begin{table}[!ht]
    \centering
    \small
    \setlength{\tabcolsep}{0.11cm}
    \caption{\textbf{Quantitative results in autoregressive language modeling.} The best result is highlighted with \textbf{bold} and the second with \underline{underlined}. The smaller the better for the PPL metric. LS stands for transformer-LS. 
    }
    \label{lm:res}
    \begin{tabular}{lccc}
    \hline
        Method & PPL (val) & PPL (test) & Params (m) \\ 
        \hline
        Vanilla & \underline{29.63} & \textbf{31.01} & 156.00 \\
        \hline
        LS & 32.37 & 32.59 & 159.46 \\ 
        FLASH-quad & 31.88 & 33.50 & 153.51 \\
        FLASH & 33.18 & 34.63 & 153.52 \\ 
        1+elu & 32.63 & 34.25 & 156.00 \\ 
        Performer & 75.29 & 77.65 & 156.00 \\ \hline
        \formername T1 & 29.89 & \underline{31.35} & 155.99 \\ 
        \formername T2 & \textbf{29.57} & \textbf{31.01} & 155.99 \\ \hline
    \end{tabular}
\end{table}

\begin{table*}[!ht]
    \centering
    \setlength{\tabcolsep}{0.3cm}
    \small
    \caption{\textbf{Quantitative results of the GLUE benchmark.} MNLI is reported by the match/mismatch splits. MRPC is reported by F1 score. CoLA is reported by Matthews correlation coefficient. All the other tasks are measured by the accuracy. LS stands for transformer-LS. The best result is highlighted with \textbf{bold} and the second with \underline{underlined}. The larger the better for all metrics. "-" means unconverged.}
    \label{glue:res}
    \begin{tabular}{l|ccccccccc}
    \hline
        Method &  MNLI &  QNLI &  QQP &  SST-2 &  MRPC &  CoLA & AVG & Params (m)\\ \hline
        Vanilla & 79.37/79.07 & 87.79 & 88.04 & 90.25 & 88.35 & 38.63 & \underline{78.79} & 124.70\\
        \hline
        FLASH-quad& 78.71/79.43 & 86.36 & 88.95 & 90.94 & 81.73 & 41.28 & 78.20 &127.11\\ 
        FLASH & 79.45/{80.08} & 87.10 & 88.83 & 90.71 & 82.50 & 29.40 & 76.87 &127.12\\ 
        LS & 77.01/76.78 & 84.86 & 86.85 & 90.25 & 82.65 & 40.65 & 77.01 & 128.28\\ 
        Performer & 58.85/59.52 & 63.44 & 79.10 & 81.42 & 82.11 & 19.41 & 63.41 & 124.70\\ 
        1+elu & 74.87/75.37 & 82.59 & 86.9 & 87.27 & 83.03 & - & 70.00 & 124.0\\ 
        \hline
        \formername T1 & 79.06/79.93 & 87.00 & 88.61 & 91.17 & 84.50 & 45.38 & \textbf{79.38} & 124.67  \\ 
        \formername T2 & 77.28/78.53 & 85.39 & 88.56 & 90.71 & 85.06 & 45.90 & 78.78 & 124.67\\ \hline
    \end{tabular}
\end{table*}

\begin{table*}[!ht]
    \centering
    \setlength{\tabcolsep}{0.5cm}
    \small
    \caption{\textbf{Quantitative results on the Long-Range Arena benchmark.} The best result is highlighted with \textbf{bold} and the second with \underline{underlined}. The larger the better for all metrics.}
    \begin{tabular}{l|llllll}
    \hline
        Model & Text & ListOps & Retrieval & Pathfinder & Image & AVG. \\ \hline
        Transformer & 61.95 & 38.37 & 80.69 & 65.26 & 40.57 & 57.37 \\ \hline
        Kernelized Attention & 60.22 & 38.78 & 81.77 & 70.73 & 41.29 & 58.56 \\ 
        Nystromformer & 64.83 & 38.51 & 80.52 & 69.48 & 41.30 & 58.93 \\ 
        Linformer & 58.93 & 37.45 & 78.19 & 60.93 & 37.96 & 54.69 \\ 
        Informer & 62.64 & 32.53 & 77.57 & 57.83 & 38.10 & 53.73 \\ 
        Performer & 64.19 & 38.02 & 80.04 & 66.30 & 41.43 & 58.00 \\ 
        Reformer & 62.93 & 37.68 & 78.99 & 66.49 & 48.87 & 58.99 \\ 
        BigBird & 63.86 & 39.25 & 80.28 & 68.72 & 43.16 & 59.05 \\ 
        Skyformer & 64.70 & 38.69 & 82.06 & 70.73 & 40.77 & 59.39 \\ 
        LS & 66.62 & 40.30 & 81.68 & 69.98 & 47.60 & 61.24 \\ 
        cosFormer & \underline{67.70} & 36.50 & 83.15 & 71.96 & \underline{51.23} & 62.11 \\
        FLASH-quad & 64.10 & \textbf{42.20} & 83.00 & 63.28 & 48.30 & 60.18 \\ 
        FLASH & 64.10 & 38.70 & \textbf{86.10} & 70.25 & 47.40 & 61.31 \\ 
        \hline
        \formername T1 & 66.90 & 41.03 & 83.11 & \underline{75.92} & \textbf{51.60} & \underline{63.71} \\ 
        \formername T2 & \textbf{72.20} & \underline{41.60} & \underline{83.82} & \textbf{76.80} & 49.60 & \textbf{64.80} \\ \hline
    \end{tabular}
    \label{table: lra res}
\end{table*}

\subsection{Results}
\paragraph{Autoregressive language modeling}
\label{sec:alm}
We report the results in Table \ref{lm:res}. It can be found that both \formername variants get comparable or better perplexity to the vanilla attention and outperform all existing linear models with a clear margin. 
For example, compared to previous state-of-the-art linear methods on validation set\cite{hua2022transformer} and test set\cite{zhu2021longshort}, \formername T2 achieves substantially lower perplexity by 2.31 and 1.58 respectively. 
It demonstrates the effectiveness of our method in causal models.  


\begin{table*}[!ht]
    \centering     
    \setlength{\tabcolsep}{0.25cm}
    \small
    \caption{\textbf{Speed comparison on Long-Range Arena benchmark.} We mark it with a dash if a method  exhausts GPU memory. The higher the better for all metrics. The \textbf{1K},...,\textbf{5K} represent the input sequence length.}
    \begin{tabular}{l|lllll|lllll} 
    \hline                
    & \multicolumn{5}{c|}{Inference Speed(steps per sec)} & \multicolumn{5}{c}{Train Speed(steps per sec)}  \\ \hline 
    model & \textbf{1K} & \textbf{2K} & \textbf{3K} & \textbf{4K} & \textbf{5K} & \textbf{1K} & \textbf{2K} & \textbf{3K} & \textbf{4K} & \textbf{5K} \\ \hline 
    Transformer    & 39.06 & 10.05 & -  & -  & -   & 15.34  & 3.05 & -  & -   & -  \\ \hline
    FLASH-quad     & 44.64 & 16.45 & 9.40 & 6.54 & 5.39 & 19.84 & 8.47 & 5.19 & 3.59 & 2.92 \\ 
    FLASH          & 40.32 & 23.15 & 16.89 & 14.04 & 13.16 & 20.49 & 11.06 & 8.47 & 7.23 & 6.93 \\ 
    LS      & 32.05 & 17.36 & 12.14 & 10.16 & 9.06 & 15.43 & 8.68 & 6.28 & 5.24 & 4.76  \\ 
    Performer      & 104.17 & 56.82 & 42.37 & 33.78 & 31.25 & 28.41 & 16.23 & 12.02 & 10.04 & 9.06 \\ 
    cosFormer      & 86.21 & 46.30 & 32.47 & 27.47 & 25.00 & 22.94 & 12.82 & 9.19 & 7.79 & 7.14  \\ 
    Linformer      & 104.17 & 58.14 & 40.32 & 31.25 & 26.32 & 27.17 & 15.63 & 11.26 & 8.77 & 7.42 \\ 
    Reformer       & 78.13 & 38.46 & 26.04 & 19.84 & 16.23 & 20.16 & 10.87 & 7.46 & 5.69 & 4.70 \\ 
    Nystorm        & 58.14 & 38.46 & 29.07 & 23.81 &  20.33  & 14.12 & 9.62 & 7.46 & 6.11 & 5.26 \\ \hline
    \formername T1 & 113.64 & 65.79 & 46.30 & 39.06 & 35.71 & 28.41 & 17.12 & 12.76 & 10.87 & 10.12 \\
    \formername T2 & 119.05 & 65.79 & 47.17 & 39.68 & 36.23 & 29.41 & 17.24 & 12.95 & 10.96 & 10.16  \\ 
    \hline
    \end{tabular} 
    \label{speed:res}
\end{table*}

\paragraph{Bidirectional language modeling}
\label{sec:blm}
We show our bidirectional results on the GLUE benchmark in Table.~\ref{glue:res}.
Our method achieves superior performance to all the competing methods in average. 
On three tasks, \ie SST-2, MRPC, CoLA, \formername reports comprehensively better results than all competing linear methods, such as 4.62 higher on CoLA.
Further, one of our variants \ie \formername T1,  even outperforms the vanilla attention with a notable margin. It proves the effectiveness of our method in bidirectional language modeling.
\paragraph{Long Range Arena Benchmark}
\label{sec:lra}
The results before the transformer Long-short (abbr. LS) are taken from the Skyformer~\cite{Skyformer}.
As shown in Table.~\ref{table: lra res}, we achieve either first or second places across all five tasks. 
In terms of overall results, both \formername variants (T1,T2)
outperform all other competing methods including vanilla transformer~\cite{vaswani2017attention}, which validates our capability to encode long sequences.

\subsection{Speed comparison}
We compare the training and inference speed of the \formername with other methods. 
For a fair and comprehensive comparison, we follow exactly the same configurations of the Skyformer\cite{Skyformer} and report step per second under different sequence lengths. Timing is conducted on a Nvidia A6000 GPU with 48G GPU memory. 
Table.~\ref{speed:res} suggests that the vanilla transformer is substantially slow and exhausts GPU memory with sequence longer than 3k.
Compared to other efficient transformers, our \formername achieves faster speed with comparable GPU memory footprints, 
while competing efficient methods all report worse results compared to our \formername.
For instance, compared to FLASH-quad~\cite{hua2022transformer} that achieves previous best linear results on both autoregressive and bidirectional benchmarks, our model performs over 300\% faster during training and 150\% faster during inference.

\subsection{Ablation study}
In this section, we justify our design choice of the \formername, including , the selection of the FFN module, and the size of the attention block in \localattentionname. We use the PPL from the Roberta pre-training stage as our evaluation metric.


\paragraph{Structure design}
\begin{table}[h]
    \centering
    \setlength{\tabcolsep}{0.2cm}
    \small
    \vspace{-1mm}
    \caption{\textbf{Ablation of the proportion of the attentions.} We empirically find that the balanced structure achieves the best result. We abbreviate the \localattentionname as BlockAtt and \attentionname as NormAtt.}
    \begin{tabular}{c c|c c}
    \hline
        \makecell{Early stage \\ BlockAtt} & \makecell{Later stage \\ NormAtt} & T1 ppl(val)  & T2 ppl(val)\\ \hline
        0 & 12 & 4.23 & 4.48\\ 
        3 & 9 & 4.13 & 3.83\\ 
        6 & 6 & \textbf{3.82} & \textbf{3.81}\\
        9 & 3 & 3.87 & 3.86\\ 
        12 & 0 & 4.75 & 4.66\\
         \hline
    \end{tabular}
    \vspace{-1mm}
    \label{table:percent_v1}
\end{table}
\begin{table}[h]
    \caption{\textbf{Ablation of the order of two proposed attention.} Using \localattentionname in the early stage achieves better results than using it on later stage.}
    \centering
    \setlength{\tabcolsep}{0.2cm}
    \small
    \begin{tabular}{ll|cc}
    \hline
        Early stage & Later stage & T1 ppl(val) & T2 ppl(val) \\ \hline
        NormAtt & BlockAtt & 4.13 &  4.21 \\ 
        BlockAtt & NormAtt & \textbf{3.82} & \textbf{3.81} \\ \hline
    \end{tabular}
    \label{order_v1}
\end{table}
As aforementioned, we empirically choose the first 6 layers as the early stage of the model and the rest as the later stage. We provide the designing ground for this choice in Table.~\ref{table:percent_v1}. It can be also observed that either choosing the \localattentionname or \attentionname for the entire model will lead to inferior performance. We also provide the ablation results of swapping the order of the  \localattentionname and the \attentionname in Table.~\ref{order_v1}. Using \localattentionname in the early stage achieves significantly better results than using it on later stage. It further proves our claim that the early stage focuses on neighbouring tokens while the later stage needs long-range attentions.
\paragraph{FFN module}
\begin{table}[h]
\vspace{-3mm}
    \caption{\textbf{Ablation of the selection of the FFN modules.} The GLU leads to better results.}
    \label{ffn_v1}
    \centering
    \setlength{\tabcolsep}{0.5cm}
    \small
    \begin{tabular}{l|cc}
    \hline
        FFN type & T1 ppl(val) & T2 ppl(val)\\ \hline
        FFN & 3.93 & 3.93\\ 
        GLU(ours) & \textbf{3.82} & \textbf{3.81} \\ \hline
    \end{tabular}
\end{table}

We ablate the selection of the FFN modules in Table.~\ref{ffn_v1}. Compared with the traditional FFN~\citep{vaswani2017attention}, the GLU~\citep{2002.05202} achieves better results.

\paragraph{Block size}
\begin{table}[h]
    \caption{\textbf{Ablation of on block sizes in the \localattentionname.} The larger block size the better results.}
    \label{chunk_v1}
    \centering
    \setlength{\tabcolsep}{0.45cm}
    \small
    \begin{tabular}{c|cc}
    \hline
        Block size & T1 ppl(val) & T2 ppl(val)\\ \hline
        32 & 3.92 &  3.90 \\ 
        64 & 3.82 & 3.81 \\
        128 & \textbf{3.72} & \textbf{3.69} \\ 
         \hline
    \end{tabular}
\end{table}
From the Table.~\ref{chunk_v1}, we observe clear performance improvements with increased block sizes. 
However, since the complexity of the \localattentionname is $O(nwd)$,  larger block size $w$ leads to heavier computational overhead.
We choose a block size as 64 as a trade-off between performance and computational cost.

\paragraph{Combination of attentions}
Finally, we study the effect that whether we should use both attentions in one layer.
In particular, we compare either to 1) use \localattentionname and \attentionname sequentially in a layer with different orders; or to 2) use them in parallel in each attention layer and then concatenate their embedding output.
Table.~\ref{approch_v1} shows that we should not use these attentions sequentially within a layer and apply them in parallel will double the computation complexities without improving the performance.
\begin{table}[h]
    \caption{\textbf{Ablation of the combination of two proposed attention.} In first two rows, the two attention layers appear in an interleaved manner. D for the \localattentionname and N for the \attentionname.}
    \centering
    \setlength{\tabcolsep}{0.25cm}
    \small
    \begin{tabular}{l|cc}
    \hline
        approach & T1 ppl(val) & T2 ppl(val)\\ \hline
        altering D$\rightarrow$N & 4.19 & 4.23\\ 
        altering N$\rightarrow$D & 4.11 & 4.21 \\ 
        parallel & \textbf{3.77} & 3.82\\ 
        \formername & 3.82  & \textbf{3.81}\\ \hline
    \end{tabular}
    \label{approch_v1}
\end{table}

\section{Conclusion}
In this paper, we identified two key issues that cause the inferior performance of existing linear transformer models: 1) unbounded gradients; 2) attention dilution. For the former issue, we proposed a new \attentionname to stabilize the training gradients. For the latter, we develop \localattentionname to force the model concentrate attention in neighbouring tokens.
The resultant model \formername marries the strength of the vanilla transformers and the linear transformers, outperforming competing linear transformers on both autoregressive and bidirectional language modeling, text classification tasks and the challenging Long-range arena benchmark.

\section*{Limitations}
In this paper, we identified two main issues of current linear transformers and provided a comprehensive analysis in natural language processing tasks. However, with the booming development of vision transformers, whether they share the same issues of linear NLP transformers is yet to be discovered. We will validate our method on the linear vision transformers in our future work.  


\section*{Ethics Statement}
The proposed technique is beneficial to develop large-scale environment-friendly language models by reducing computing resource demand.
Corpus used to train the model is from public web sources, which may contain biased, explicit or improper content.
Further assessment and regulation have to be in-place before deploying the model in practice.


\bibliography{anthology,custom}
\bibliographystyle{acl_natbib}

\appendix
\noindent\large{\textbf{Appendix}}
\section{Mathematical Notations}
We use bold uppercase letters for matrices($\mathbf M$), bold lowercase letters for vectors($\mathbf m$), and lowercase letters for scalars($m_{ij}$).
We represent all vectors as column vectors and denote the $i$th row of matrix $\mathbf M$ by $\mathbf m_i^\top$ or $\mathbf M_i$. We use $\|.\|_2$ to denote the $l_2$ norm and $\|.\|_F$ to denote the Frobenius norm of the matrix and the vector.

The main mathematical symbols are input $\mathbf X \in \mathbb R^{n\times d}$, $\mathbf Q$ (\textbf{Query}), $\mathbf K$ (\textbf{Key}) 
and $\mathbf V$ (\textbf{Value}), which has the following form:
\begin{equation}
\begin{aligned}
\mathbf{X}
&=\left[
 \begin{matrix}
\mathbf{x}_1\tran \\
\vdots\\
\mathbf{x}_n\tran
  \end{matrix}
  \right] \in \mathbb R^{n\times d},\\
\mathbf{Q}
&=\left[
 \begin{matrix}
\mathbf{q}_1\tran\\
\vdots\\
\mathbf{q}_n\tran
  \end{matrix}
  \right] =\mathbf{XW}_Q=\left[
 \begin{matrix}
\mathbf{x}_1\tran\mathbf{W}_Q\\
\vdots\\
\mathbf{x}_n\tran\mathbf{W}_Q
  \end{matrix}
  \right] \in \mathbb R^{n\times d}, \\
 \mathbf{K}&= \left[
 \begin{matrix}
 \mathbf{k}_1\tran\\
\vdots\\
 \mathbf{k}_n\tran
  \end{matrix}
  \right] 
  =\mathbf{XW}_K=\left[
 \begin{matrix}
\mathbf{x}_1\tran\mathbf{W}_K\\
\vdots\\
\mathbf{x}_n\tran\mathbf{W}_K
  \end{matrix}
  \right] \in \mathbb R^{n\times d}, \\
\mathbf{V}&=\left[
 \begin{matrix}
\mathbf{v}_1\tran\\
\vdots\\
\mathbf{v}_n\tran
  \end{matrix}
  \right] 
  =\mathbf{XW}_V=\left[
 \begin{matrix}
\mathbf{x}_1\tran\mathbf{W}_V\\
\vdots\\
\mathbf{x}_n\tran\mathbf{W}_V
  \end{matrix}
  \right] \in \mathbb R^{n\times d}, 
\end{aligned}
\end{equation}
where $\mathbf{W}_Q,\mathbf{W}_K,\mathbf{W}_V \in \mathbb{R}^{d\times d}$.

\section{Proof of gradients' upper bound}
\label{sec:bound}
In this part, we will proof the bound in \eqref{eq:gradient1} and \eqref{eq:gradient2}, all we need to prove is:
\begin{equation}
0\le p_{ik}(1-p_{ik}) \le \frac 1 4,
0 \le p_{ij}p_{ik} \le \frac 1 4. 
\end{equation}
We adopt the theorem that geometric mean is bounded by arithmetic mean, \ie
\begin{equation}
    \sqrt{ab}\le \frac{a+b}{2}
    \Longleftrightarrow
    ab\le \left(\frac{a+b}{2}\right)^2, \forall a, b \ge 0 .
\end{equation}
We take $a=p_{ik}, b=1-p_{ik}$ to complete the proof.
The first bound can be proven by:
\begin{equation}
0\le p_{ik}(1-p_{ik})\le \left(\frac{p_{ik}+1 -p_{ik}}{2}\right)^2=\frac 1 4.
\end{equation}
For the second bound, we first use the fact that:
\begin{equation}
0\le p_{ij}+p_{ik}\le 1 \Rightarrow p_{ij} \le 1 - p_{ik}.
\end{equation}
So we have:
\begin{equation}
  0\le p_{ij }p_{ik}\le (1-p_{ik})p_{ik} \le \frac{1}{4}.
\end{equation}

\section{Proof of Proposition~\ref{gradient:linear}}
\label{sec:linear}
\begin{proof}[Proof of Proposition~\ref{gradient:linear}]
$\forall \epsilon > 0$ and kernel function $\phi $, let\footnote{We assume that the image of $\phi$ contains vectors arbitrary close to $\mathbf 0$, which is a common case in kernel function.}:
\begin{equation}
\begin{gathered}
\mathbf q_i=\mathbf k_j =\phi^{-1}(\mathbf x_0 ),\\
0<\|\mathbf x_0\|_2 \le \sqrt{\epsilon}, i,j=1,\ldots, n.
\end{gathered}
\end{equation}
Then
\begin{equation}
   \phi(\mathbf q_i)= \phi(\mathbf k_j)=\mathbf x_0, i,j =1, ..., n. 
\end{equation}
So
\begin{equation}
s_{ij}=\phi(\mathbf q_i)\tran \phi(\mathbf k_j) ={\mathbf x_0}\tran {\mathbf x_0}\in (0, \epsilon],
\end{equation}
and 
\begin{equation}
p_{ij}=\frac{s_{ij}}{\sum_{k=1}^n s_{ik}}
=\frac{{\mathbf x_0}\tran {\mathbf x_0}}{\sum_{k=1}^ n{\mathbf x_0}\tran {\mathbf x_0}}
=\frac 1 n.
\end{equation}
According to \eqref{eq:gradient2}, we have:
\begin{equation}
\begin{aligned}
 \frac{\partial p_{ij}}{\partial s_{ik}} 
&=\begin{cases}
\frac{1}{{\mathbf x_0}\tran {\mathbf x_0}}\frac 1 n(1-\frac 1 n)  & j=k\\
-\frac{1}{{\mathbf x_0}\tran {\mathbf x_0}}\frac 1 {n^2}  & j\neq k
\end{cases},\\
\left|  \frac{\partial p_{ij}}{\partial s_{ik}} \right|
&=\begin{cases}
\frac{1}{{\mathbf x_0}\tran {\mathbf x_0}}\frac 1 n(1-\frac 1 n)  & j=k\\
\frac{1}{{\mathbf x_0}\tran {\mathbf x_0}}\frac 1 {n^2}  & j\neq k
\end{cases}\\
&\ge \begin{cases}
\frac{1}{\epsilon n}(1-\frac 1 n)  & j=k\\
\frac{1}{\epsilon n^2}  & j\neq k
\end{cases}.
\end{aligned}
\end{equation}
Let $\epsilon \to 0^+$, then $\frac 1 {\epsilon n^2},\frac 1{\epsilon n}(1-\frac 1 n)\to \infty$, so $\left| \frac{\partial p_{ij}}{\partial s_{ik}} \right|\to \infty$.

\end{proof}

\section{Analyze the gradient of each method}
\label{sec:proof}
In this section, let's consider a one-layer Transformer, for a multi-layer Transformer, we can prove our conclusion using induction.

We begin this section by introducing some mathematical notations. 

\subsection{Notations}
In vanilla attention, we have:
\begin{equation}
\begin{aligned}
\mathbf{S}&=\mathbf Q \mathbf K\tran \in \mathbb R^{n\times n} , \\
\mathbf{P}&= \mathrm{Softmax}(\mathbf S) \in \mathbb R^{n\times n} , \\
\mathbf O &= \mathbf{P} \mathbf V \in \mathbb R^{n\times d} .
\end{aligned}
\end{equation}
In linear attention, we have:
\begin{equation}
\begin{aligned}
\mathbf{S}&=\phi(\mathbf Q) \phi(\mathbf K)\tran \in \mathbb R^{n\times n},\\
\mathbf{\Delta} &= \mathrm{diag}(\mathbf S {\mathbf 1}_n) \in \mathbb R^{n\times n},\\
\mathbf {P} &= \mathbf \Delta^{-1} \mathbf{S}  \in \mathbb R^{n\times n},\\
\mathbf{O}&=\mathbf {P} \mathbf V \in \mathbb R^{n\times d}.\\
\end{aligned}
\end{equation}
Although this term is not calculated in linear attention, we discuss it conceptually.
Note that the above formulations can be unified into the following form \footnote{Here, the function $f(\mathbf X)$ is applied element-wise to the matrix $\mathbf X\in \mathbb R^{n\times m}$, that is, $[f(\mathbf X)]_{ij}=[f(x_{ij})]$}:
\begin{equation}
\label{vanilla_linear}
\begin{aligned}
\mathbf{S}&=f(\psi(\mathbf Q) \psi(\mathbf K)\tran)  \in \mathbb R^{n\times n},\\
\mathbf{\Delta} &= \mathrm{diag}(\mathbf S {\mathbf 1}_n)  \in \mathbb R^{n\times n},\\
\mathbf {P} &= \mathbf \Delta^{-1} \mathbf{S}  \in \mathbb R^{n\times n},\\
\mathbf{O}&=\mathbf {P} \mathbf V \in \mathbb R^{n\times d},
\end{aligned}
\end{equation}
where in vanilla attention, we have:
\begin{equation}
\begin{aligned}
\psi(\mathbf x)=\mathbf x, f(x)= \exp(x), 
\end{aligned}
\end{equation}
and in linear attention, we have:
\begin{equation}
\begin{aligned}
\psi(\mathbf x)=\phi(\mathbf x) , f(x)= x. 
\end{aligned}
\end{equation}

In \attentionname, we have:
\begin{equation}
\begin{aligned}
\mathbf{S}&=\phi(\mathbf Q) \phi(\mathbf K)\tran  \in \mathbb R^{n\times n},\\
\mathbf{T}&= \mathbf S \mathbf V  \in \mathbb R^{n\times d},\\
\mathbf O&= \mathrm{RMSNorm} (\mathbf{T})\\
&\triangleq 
\left[
 \begin{matrix}
 \mathrm{RMSNorm}(\mathbf{t}_1)\tran \\
\vdots\\
 \mathrm{RMSNorm}(\mathbf{t}_n)\tran
  \end{matrix}
  \right] 
\in \mathbb R^{n\times d},
\end{aligned}
\end{equation}
where $\mathrm{RMSNorm}$ is defined as follows:
\newtheorem{def:rms}{Definition}[section]
\begin{def:rms}
\begin{equation}
\begin{aligned}
\mathrm{RMSNorm}(\mathbf x) 
&= \frac{\mathbf x}{\sqrt{\sigma^2 +\epsilon}}, \\
\sigma^2&= \frac{\sum_{i=1}^d x_i^2}{d} ,\\
\epsilon & > 0, \\
\mathbf x &\in \mathbb R^d.
\end{aligned}
\end{equation}
\end{def:rms}
In the subsequent discussion, we define gradient $\nabla_{\mathbf M} \mathcal L$ as:
\newtheorem{def:grad}[def:rms]{Definition}
\begin{def:grad}
\begin{equation}
[\nabla_{\mathbf M} \mathcal L]_{ij}=\frac{\partial \mathcal L}{\partial m_{ij}},
\end{equation}
where $\mathcal L$ stands for loss function, $\mathbf M$ is a parameter matrix.
\end{def:grad}
Then we define the mapping $h$ as:
\newtheorem{def:h}[def:rms]{Definition}
\label{def:h}
\begin{def:h}
\begin{equation}
\begin{gathered}
h:  \mathbb R^{n\times m} \to \mathbb R, h(\mathbf X)=\max_{i=1}^n \| \mathbf X_i\|_2,\\
\mathbf X\in \mathbb R^{n\times m}.
\end{gathered}
\end{equation}
\end{def:h}
The mapping $h$ has the following property:
\newtheorem{prop:h}[def:rms]{Proposition}
\label{prop:h}
\begin{prop:h}
$\forall \mathbf X\in \mathbb R^{n\times m}, \mathbf Y \in \mathbb R^{r\times m}$, we have:
\begin{equation}
h(\mathbf X\mathbf Y^{\top}) \le \sqrt r h(\mathbf X)h (\mathbf Y).
\end{equation}
\end{prop:h} 

\begin{proof}[Proof]
\label{proof1}
Since
\begin{equation}
\begin{aligned}
[\mathbf X\mathbf Y^{\top}]_{ij}
&=\mathbf  X_i [\mathbf Y_j]^{\top}\\
&\le {\| \mathbf X_i\|_2}\|\mathbf Y_j\|_2\\
&\le h(\mathbf X)h(\mathbf Y),
\end{aligned}
\end{equation}
so
\begin{equation}
\begin{aligned}
\| [\mathbf X\mathbf Y^{\top}]_i\|_2
&=\sqrt{\sum_{j=1}^r ([\mathbf X\mathbf Y^{\top}]_{ij})^2} \\
&\le \sqrt{r(h(\mathbf X) h(\mathbf Y))^2}\\
&=\sqrt r h(\mathbf X)h(\mathbf Y),\\
h(\mathbf X\mathbf Y^{\top})
&=\max_{i=1}^r \left\| [\mathbf X\mathbf Y^{\top}]_i \right\|_2 \\
&\le \sqrt r h(\mathbf X)h(\mathbf Y).
\end{aligned}
\end{equation}
\end{proof}

\subsection{Gradient analysis}
\subsubsection{Preliminary}
Given gradient $\nabla_{\mathbf O}\mathcal L \in \mathbb R^{n\times d}$, let's compute $\nabla_{ \mathbf S}\mathcal L $ in every situation. 

We first define:
\begin{equation}
\label{upperbound}
\begin{aligned}
c_1&=h(\nabla_{\mathbf O}\mathcal L)\\
&= \max_{i=1}^n \|\nabla_{\mathbf O_i}\mathcal L \|_2,  \\
c_2&=h(\mathbf V)\\
&= \max_{i=1}^n \|\mathbf V_i \|_2 <\infty, \\
c_3&=  \min_{i,j} | s_{ij}| \ge 0.
\end{aligned}
\end{equation}
Before we get started, we have the following propositions. The proof can be found in Appendix \ref{proof:hy}.
\newtheorem{c1}[def:rms]{Proposition}
\begin{c1}
\label{prop:c1}
$c_1 < \infty$.
\end{c1} 
\newtheorem{th1}[def:rms]{Proposition}
\begin{th1}
\label{prop:th1}
$\forall \mathbf X\in \mathbb R^{n\times m}$, we have:
\begin{equation}
\|\mathbf X\|_2 \le \sqrt n h(\mathbf X).
\end{equation}
\end{th1}

Take $\mathbf X= \mathbf V$, we get:
\begin{equation}
\|\mathbf V\|_2 
\le \sqrt{n} h(\mathbf V) =\sqrt{n} c_2.
\end{equation}

\subsubsection{Vanilla/Linear attention}
According to \eqref{vanilla_linear}, we can discuss vanilla and linear attention under one formula:
\begin{equation}
\nabla_{\mathbf P} \mathcal L
=[\nabla_{\mathbf O}  \mathcal L]\mathbf V\tran \in \mathbb R^{n\times n}.
\end{equation}
Then define matrix $\mathbf U^{(i)}\in \mathbb R^{n\times n}$:
\begin{equation}
[\mathbf U^{(i)}]_{jk} = \frac{\partial p_{ik}}{\partial s_{ij}}.
\end{equation}
According to \eqref{vanilla}, in vanilla attention, we have:
\begin{equation}
\left|[\mathbf U^{(i)}]_{jk} \right| \le \frac{1}{4},
\end{equation}
while in linear attention, we have:
\begin{equation}
\left|[\mathbf U^{(i)}]_{jk} \right| \le \frac{1}{4|s_{ij}|} \le \frac 1 {4 c_3}.
\end{equation}
Since:
\begin{equation}
\begin{aligned}
\frac{\partial \mathcal L}{\partial s_{ij}}
&=\sum_{k=1}^n\frac{\partial \mathcal L}{\partial p_{ik}}\frac{\partial p_{ik}}{\partial s_{ij}}\\
&= (\nabla_{\mathbf P_i}  \mathcal L) (\mathbf U^{(i)}_j)^\top\\
&= (\nabla_{\mathbf O_i}  \mathcal L)\mathbf V\tran (\mathbf U^{(i)}_j)^\top .
\end{aligned}
\end{equation}
So we have:
\begin{equation}
\begin{aligned}
\left|\frac{\partial \mathcal L}{\partial s_{ij}} \right|
& \le \| (\nabla_{\mathbf O_i}\mathcal L )\mathbf V\tran\|_2 \left \| {\mathbf U^{(i)}_j}\tran \right\|_2\\
&\le \|\nabla_{\mathbf O_i}\mathcal L \|_2  \| \mathbf V\tran\|_2  \|\mathbf U^{(i)}_j \|_2\\
&\le c_1\times  \sqrt{n}c_2 \times \frac 1  {4 t} \\
&=\frac{ \sqrt{n}c_1 c_2}{4t},
\end{aligned}
\end{equation}
where $t=1$ in vanilla attention and $t=c_3$ in linear attention.

On the other hand, according to Appendix ~\ref{sec:linear}, in linear attention, there exist $\mathbf q_i, \mathbf k_j$, such that:
\begin{equation}
\begin{aligned}
\frac{\partial p_{ik}}{\partial s_{ij}}&=\frac{1}{\|{\mathbf x_0}\tran {\mathbf x_0}\|} t_{ijk},\\
t_{ijk}&=\begin{cases}
\frac 1 n(1-\frac 1 n)  & j=k\\
-\frac 1 {n^2}  & j\neq k
\end{cases}.
\end{aligned}
\end{equation}
Then
\begin{equation}
\begin{aligned}
\left| \frac{\partial \mathcal L}{\partial s_{ij}} \right |
&=\left|\sum_{k=1}^n\frac{\partial \mathcal L}{\partial p_{ik}}\frac{\partial p_{ik}}{\partial s_{ij}}  \right|\\
&= \frac{1}{\|{\mathbf x_0}\tran {\mathbf x_0}\|} \left|\sum_{k=1}^n\frac{\partial \mathcal L}{\partial p_{ik}} t_{ijk} \right|\\
&\ge \frac 1 \epsilon \left| \sum_{k=1}^n\frac{\partial \mathcal L}{\partial p_{ik}} t_{ijk} \right|.
\end{aligned}
\end{equation}
Let $\epsilon \to 0^+$, then $\left| \frac{\partial\mathcal L}{\partial s_{ik}} \right|\to \infty$. This means that the gradient in linear attention is unbounded.

\subsubsection{\attentionname}
We first define the second-moment of $i$'th row of $\mathbf T$:
\begin{equation}
\begin{aligned}
\sigma_i^2&= \frac{\sum_{j=1}^d t_{ij}^2}{d}.
\end{aligned}
\end{equation}
Then $\frac{\partial o_{ij}}{\partial t_{ik}}$ is as follows:
\begin{equation}
\begin{aligned}
\frac{\partial o_{ij}}{\partial t_{ik}}
&=\frac 1{\sqrt{\sigma_i^2 + \epsilon}}
\left[{ 1\{j=k\} }  -\frac 1 d \frac{t_{ij}t_{ik}}{\sigma_i^2 + \epsilon}
\right ].
\end{aligned}
\end{equation}
Notice that we have the following upper bound:
\begin{equation}
\label{rmsuppper}
\begin{aligned}
&\left|\frac{\partial o_{ij}}{\partial t_{ik}}\right|\\
= & \frac 1{\sqrt{\sigma_i^2 + \epsilon}} 
\left[1\{j=k\} 
-\frac 1 d \frac{t_{ij} t_{ik} 
}{\frac{\sum_{s=1}^d t_{is}^2}  d+ \epsilon} \right ] \\
= & \frac 1{\sqrt{\sigma_i^2 + \epsilon}} 
\left[1\{j=k\} 
+ \frac{t_{ij} t_{ik} 
}{\sum_{s=1}^d t_{is}^2 +d\epsilon} \right ] \\
\le& \frac 1{\sqrt{\sigma_i^2 + \epsilon}} 
\left[1\{j=k\} 
+\frac 1 2 \frac{t_{ij}^2 
+t_{ik}^2}{\sum_{s=1}^d t_{is}^2} \right ] \\
\le & \frac 1{\sqrt{\sigma_i^2 + \epsilon}} 
\left[1 + \frac 1 2 \right]\\
\le & \frac 3{2\sqrt{\sigma_i^2 + \epsilon}}.
\end{aligned}
\end{equation}
Let's define matrix $\mathbf  R^{(i)} \in \mathbb R^{d\times d}$ as follows:
\begin{equation}
\begin{aligned}
[\mathbf  R^{(i)}]_{jk}&=\frac{\partial o_{ik}}{\partial t_{ij}}.
\end{aligned}
\end{equation}
Since
\begin{equation}
\begin{aligned}
\frac{\partial \mathcal L}{\partial t_{ij}}
&=\sum_{k=1}^n\frac{\partial \mathcal L}{\partial o_{ik}}\frac{\partial o_{ik}}{\partial t_{ij}}\\
&= (\nabla_{\mathbf O_i}  \mathcal L) (\mathbf R^{(i)}_j)^\top.
\end{aligned}
\end{equation}
Then we can get:
\begin{equation}
{\nabla_{\mathbf T_i}}\mathcal L  = 
({\nabla_{\mathbf O_i}}\mathcal L)
(\mathbf R^{(i)})\tran\in \mathbb R^{1\times d}.
\end{equation}
According to \eqref{rmsuppper}, we have:
\begin{equation}
\begin{aligned}
\|\mathbf R^{(i)}\|_2 
&\le  \|\mathbf R^{(i)}\|_F \\
&\le \sqrt{\sum_{j=1}^d \sum_{k=1}^d \left[\frac{\partial o_{ij}}{\partial t_{ik}} \right]^2}\\
&\le \frac {3d}{2\sqrt{\sigma_i^2 + \epsilon}} \\
&\le \frac{3d}{2\sqrt{\epsilon}}.
\end{aligned}
\end{equation}
Finally, we get:
\begin{equation}
\begin{aligned}
\nabla_{\mathbf S_i} \mathcal L
&=(\nabla_{\mathbf T_i}  \mathcal L)\mathbf V\tran \\
&=({\nabla_{\mathbf O_i}}\mathcal L)
(\mathbf R^{(i)})\tran \mathbf V\tran \in \mathbb R^{1\times n},\\
\frac{\partial \mathcal L}{\partial s_{ij}}
&= (\nabla_{\mathbf O_i}  \mathcal L)
(\mathbf R^{(i)})\tran \mathbf V_j, \\
\left| \frac{\partial \mathcal L}{\partial s_{ij}} \right|
&=\left|(\nabla_{\mathbf O_i}  \mathcal L) 
(\mathbf R^{(i)})\tran \mathbf V_j  \right|\\
& \le \| \nabla_{\mathbf O_i}  \mathcal L \|_2  \| \mathbf R^{(i)} \mathbf V_j\|_2 \\
&\le  \| \nabla_{\mathbf O_i}  \mathcal L \|_2  \| \mathbf R^{(i)}\|_2 \| \mathbf V_j\|_2 \\
& \le c_1 \times  \frac{3d}{2\sqrt{\epsilon}} \times  c_2 \\
&=  \frac{3c_1 c_2d}{2\sqrt{\epsilon}}.
\end{aligned}
\end{equation}
Let's summarize the previous results.\\
In vanilla attention, we have:
\begin{equation}
\left|\frac{\partial \mathcal L}{\partial s_{ij}} \right| \le \frac{\sqrt{n}c_1 c_2}{4} < \infty.
\end{equation}
In linear attention,  there exist $\mathbf q_i, \mathbf k_j$, such that:
\begin{equation}
\left|\frac{\partial \mathcal L}{\partial s_{ij}} \right| \to \infty. 
\end{equation}
In \attentionname, we have:
\begin{equation}
\left|\frac{\partial \mathcal L}{\partial s_{ij}} \right| \le \frac{3c_1 c_2d}{2\sqrt{\epsilon}} <\infty.
\end{equation}
So $\frac{\partial \mathcal L}{\partial s_{ij}}$ is bounded in vanilla attention and \attentionname, while it's unbounded in linear attention. This makes the training of linear transformer unstable.

\subsection{Proof of the proposition}
\label{proof:hy}

\begin{proof}[Proof of Proposition~\ref{prop:c1}]
Let's consider a one layer Transformer for classification tasks. The input is $\mathbf X\in \mathbb R^{n\times d}$, the label is $\mathbf Y \in \mathbb R^{n\times m}$, where $m$ is the number of categories and $\mathbf Y_i$ is one-hot vector,. $f_1,f_2$ are the activation functions, here we take $f_1=f_2=\mathrm{ReLU}$ as an example. The parameters of the model are:
\begin{equation}
\begin{gathered}
\mathbf W_1\in \mathbb R^{d\times d_1}, \mathbf W_2\in \mathbb R^{d_1 \times d},\\
\mathbf W_{3}\in \mathbb R^{d\times m}.
\end{gathered}
\end{equation}
The forward pass of the model is\footnote{XAttention stands for vanilla/norm attention.}: 
\begin{itemize}
\item $\mathbf X_1 = \mathrm {XAttention}(\mathbf X)\in \mathbb R^{n\times d}. $
\item  $\mathbf X_2=f_1(\mathbf X_1 \mathbf W_1)\in \mathbb R^{n\times d_1}.$
\item  $\mathbf X_3=f_2(\mathbf X_2\mathbf W_2)\in \mathbb R^{n\times d}.$
\item  $\mathbf O= \mathbf T\mathbf  W_{3}\in \mathbb R^{n\times m}.$
\item  $\mathbf P=\mathrm{Softmax}(\mathbf O)\in \mathbb R^{n\times m}.$
\item  $\mathcal L= \mathrm{corss\_entropy}(\mathbf P,\mathbf  Y)\in \mathbb R .$
\end{itemize}
The backward pass of the model is:
\renewcommand{\outlineii}{enumerate}
\begin{outline}[enumerate]
\1 $\nabla_{\mathbf O }{ \mathcal L} =\mathbf P-\mathbf Y \in \mathbb R^{n\times m}.$
\2  The upper bound is:
\begin{equation}
\begin{aligned}
 & h\left(\nabla_{ \mathbf O }{ \mathcal L} \right) \\
=&\max\{\sum_{i=1}^m p_i^2 -2p_1 + 1,\\
&p_i\ge 0, \sum_{i=1}^m p_i = 1  \} \\
\triangleq& a_0  \\
<& \infty .
\end{aligned}
\end{equation}
\1  $\nabla_{\mathbf X_3}{\mathcal L}=(\nabla_{\mathbf O }{ \mathcal L})W_{3}^{\top} \in \mathbb R^{n\times d} .$
\2  The upper bound is:
\begin{equation}
  \begin{aligned}
 & h\left(\nabla_{\mathbf X_3}{\mathcal L} \right)\\
 \le &\sqrt{d}h\left(\nabla_{\mathbf O} {\mathcal L}\right )h\left( \mathbf W_{3}\right)\\
\le  &\sqrt d a_0 h\left(\mathbf  W_{3}\right)\\
 \triangleq & a_1 < \infty .
  \end{aligned} 
\end{equation}

\1  $\nabla_{\mathbf X_2}{ \mathcal L}=\left(f_2'(\mathbf X_2 \mathbf W_2) \odot {\nabla_{\mathbf X_3} \mathcal L} \right) {\mathbf W}_2^{\top}\in \mathbb R^{n\times d} .$
\2 The upper bound is:
\begin{equation}
\begin{aligned}
&h\left(\nabla_{\mathbf X_2}{ \mathcal L}\right)\\
\le &\sqrt{d_1} h\left(f_2'(\mathbf X_2\mathbf W_2) \odot  {\nabla_{\mathbf X_3} \mathcal L}\right) h\left( \mathbf W_2\right)\\
\le &\sqrt{d_1} a_1 h(\mathbf W_2)\\
\triangleq & a_2\\
< &\infty .
\end{aligned}
\end{equation}
\1   $\nabla_{\mathbf X_1}{ \mathcal L}=\left(f_1'(\mathbf X_1 \mathbf W_1) \odot {\nabla_{\mathbf X_2} { \mathcal L}}\right) \mathbf W_1^{\top}\in \mathbb R^{n\times d_1} .$
\2  The upper bound is:
\begin{equation}
\begin{aligned}
& h\left(\nabla_{\mathbf X_1}{ \mathcal L}\right) \\
\le & \sqrt{d} h\left(f_1'(\mathbf X_1 \mathbf W_1) \odot {\nabla_{\mathbf X_2} { \mathcal L}}\right) h\left( \mathbf W_1\right)\\
\le & \sqrt{d} a_2 h(\mathbf W_2)\\
\triangleq & a_3\\
< & \infty .
\end{aligned}
\end{equation}
\end{outline}
So the gradient passed to XAttention module is bounded, \ie $c_1=a_3 <\infty$.
\end{proof}

\begin{proof}[Proof of Proposition~\ref{prop:th1}]
\begin{equation}
\begin{aligned}
\|\mathbf X\|_2 
&\le \|\mathbf X \|_F \\
&= \sqrt{\sum_{i=1}^n \|\mathbf X_i\|_2^2} \\
&\le \sqrt{\sum_{i=1}^n [h(\mathbf X)]^2}\\
&=\sqrt{n}h(\mathbf X) .
\end{aligned}
\end{equation}
\end{proof}

\section{Experiment configs}
In this section, we will introduce detailed training hyperparameters. We introduce the configurations for autoregressive/bidirectional language model in table \ref{lm:config}. For LRA benchmark, we use the same configuration as Skyformer, which use 2-layer transformer model with 64 hidden dimensions, 2 attention heads, 85 GLU dimensions, Swish as GLU activation function. For batch size and learning rate , we use 16,1e-4 for Text Classification, 32,1e-4 for ListOps, 16,2e-4 for Document Retrieval, 128,2e-4 for Pathfinder, 256,1e-4 for Image Classification, the same as Skyformer.

\newpage
\section{Pseudocode for visualization.}
\label{appendix pseudocode}
In this section, we provide pseudo codes for the 4th column of Figure \ref{fig:attention matrix} in Python:
\begin{lstlisting}
import torch

def get_curve(w):
    n, m = w.shape
    num = 100
    P =  torch.linspace(0, 1, num)
    cnts = torch.zeros(num)
    for i in range(n):
        cnt = torch.zeros(num)
        w1 = w[i].clone()
        center = i % m
        s = w1[center].item()
        L = 1
        l = center - 1
        r = center + 1
        j = 1
        l_thre = 0
        r_thre = m
        flag = 0

        while L < m and j < num:
            if (s >= P[j].item()):
                cnt[j] = L
                j += 1
                continue
            if flag == 1:
                if r != r_thre:
                    s += w1[r].item()
                r = min(r_thre, r + 1)
                flag = 0
            else:
                if l != l_thre:
                    s += w1[l].item()
                l = max(l_thre, l - 1)
                flag = 1
            L = min(r - l + 1, m)

        if L >= m:
            for u in range(j, num):
                cnt[u] = min(L, m)

        cnt[0] = 0
        cnts += cnt
    cnts = cnts / n / m
    
    plt.plot(cnts, P)
    
    return cnts
\end{lstlisting}

\label{lm:config}
\begin{table*}[!ht]
\small
\center
\setlength{\tabcolsep}{14mm}{
\caption{Detailed configurations used in our experiments. ``Total batch size'' means $\mathrm{batch\_per\_gpu} \times \mathrm{update\_freq} \times \mathrm{num\_gpus}$. ``Attention dropout'' is only used for vanilla attention. ``ALM'': autoregressive Language Model. ``BLM'': bidirectional Language Model. }
\label{configuration}
\begin{tabular}{l|l|l}
\hline\hline
                                                             & AML & BLM  \\
\hline\hline
Data                                                         & WikiText-103                  & WikiText-103                         \\
Tokenizer method & BPE    & BPE  \\
Src Vocab size & 267744  & 50265  \\
Encoder layers                                               & 0                             & 12                                                \\
Decoder layers                                               & 6                             & 0                                                \\
Hidden dimensions                                            & 512                           & 768                                             \\
Number of heads                                              & 8                             & 12                                               \\
GLU dimensions                                               & 2048                          & 1365                                          \\
GLU activation function                                      & Swish                          & Swish                                          \\
Sequence length                                               & 512                           & 512                                               \\
Total batch size & 128                           & 512                                              \\
Number of updates                                            & 100k                           & 50k                                          \\
Warmup steps                                                 & 8k                            & 3k                                              \\
Peak learning rate                                           & 5e-4                      & 5e-4                              \\
Learning rate scheduler                                      & Inverse sqrt                  & Polynomial decay                       \\
Optimizer                                                    & Adam                          & Adam                                          \\
Adam $\epsilon$                                                     & 1e-8                      & 1e-6                                 \\
Adam $(\beta_1,\beta_2)$                                                   & (0.9, 0.98)                   & (0.9, 0.98)                            \\
Weight decay                                                 & 0.01                          & 0.01                                          \\
Gradient clipping                                            &  0.0  & 0                                               \\
Hidden dropout                                              & 0.1                           & 0.1                                             \\
Attention dropout                   & 0                             & 0.1                            \\     
\hline\hline
\end{tabular}}

\end{table*}

\end{document}